\documentclass{article}

\usepackage{microtype}
\usepackage{graphicx}
\usepackage{subfigure}
\usepackage{booktabs} 

\usepackage{hyperref}

\usepackage{amsmath}
\usepackage{amssymb}
\usepackage{mathtools}
\usepackage{caption}
\usepackage{amsthm}
\usepackage{makecell}
\usepackage{natbib}
\usepackage{bm}
\usepackage{dsfont}
\usepackage{multirow}
\usepackage{multicol}
\usepackage{hyperref}
\usepackage{url}
\usepackage{optidef}
\usepackage{relsize}
\usepackage{xcolor}
\usepackage{wrapfig,lipsum}

\definecolor{myblue}{rgb}{0.1490,0.2196,0.3804}
\definecolor{myred}{rgb}{0.6902,0.1412,0.0941}

\newcommand\scalemath[2]{\scalebox{#1}{\mbox{\ensuremath{\displaystyle #2}}}}

\usepackage[accepted]{icml2023}

\usepackage{amsmath}
\usepackage{amssymb}
\usepackage{mathtools}
\usepackage{amsthm}
\usepackage{awesomebox}
\usepackage{colortbl}

\usepackage[capitalize,noabbrev]{cleveref}

\theoremstyle{plain}
\newtheorem{theorem}{Theorem}[section]

\theoremstyle{definition}
\newtheorem{definition}[theorem]{Definition}

\theoremstyle{remark}

\newcommand{\bx}{\bm{\mathrm{x}}}
\newcommand{\bz}{\bm{\mathrm{z}}}
\newcommand{\bu}{\bm{\mathrm{u}}}

\usepackage[textsize=tiny]{todonotes}

\icmltitlerunning{Improving Adversarial Robustness of DEQs with Explicit Regulations Along the Neural Dynamics}

\begin{document}

\twocolumn[
\icmltitle{Improving Adversarial Robustness of Deep Equilibrium Models \\ with Explicit Regulations Along the Neural Dynamics}



\icmlsetsymbol{equal}{*}

\begin{icmlauthorlist}
\icmlauthor{Zonghan Yang}{thu}
\icmlauthor{Peng Li}{air,shlab}
\icmlauthor{Tianyu Pang}{sea}
\icmlauthor{Yang Liu}{thu,air,shlab}
\end{icmlauthorlist}

\icmlaffiliation{thu}{Dept. of Comp. Sci. \& Tech., Institute for AI, Tsinghua University, Beijing, China}
\icmlaffiliation{air}{Institute for AI Industry Research (AIR), Tsinghua University, Beijing, China}
\icmlaffiliation{shlab}{Shanghai Artificial Intelligence Laboratory, Shanghai, China}
\icmlaffiliation{sea}{Sea AI Lab, Singapore}

\icmlcorrespondingauthor{Peng Li}{lipeng@air.tsinghua.edu.cn}
\icmlcorrespondingauthor{Yang Liu}{liuyang2011@tsinghua.edu.cn}

\icmlkeywords{Machine Learning, ICML}

\vskip 0.3in
]



\printAffiliationsAndNotice{} 

\begin{abstract}
\looseness=-1 Deep equilibrium (DEQ) models replace the multiple-layer stacking of conventional deep networks with a fixed-point iteration of a single-layer transformation. Having been demonstrated to be competitive in a variety of real-world scenarios, the adversarial robustness of general DEQs becomes increasingly crucial for their reliable deployment. Existing works improve the robustness of general DEQ models with the widely-used adversarial training (AT) framework, but they fail to exploit the structural uniquenesses of DEQ models. To this end, we interpret DEQs through the lens of neural dynamics and find that AT under-regulates intermediate states. Besides, the intermediate states typically provide predictions with a high prediction entropy. Informed by the correlation between the entropy of dynamical systems and their stability properties, we propose reducing prediction entropy by progressively updating inputs along the neural dynamics. During AT, we also utilize random intermediate states to compute the loss function. Our methods regulate the neural dynamics of DEQ models in this manner. Extensive experiments demonstrate that our methods substantially increase the robustness of DEQ models and even outperform the strong deep network baselines.
\end{abstract}

\section{Introduction}
\label{sec1:intro}
Deep equilibrium (DEQ) models \cite{deq,mdeq} are a type of novel neural architecture. 
Different from traditional deep networks with multiple stacked layers, DEQ models explicitly cast the forward propagation as a fixed-point iteration process with a single-layer transformation:
\begin{equation}
    \bz^\star = f_{\theta}(\bz^\star; \bx),
\label{eq-1}
\end{equation}
\looseness=-1 where $f_\theta$ is the transformation parameterized with $\theta$, $\bx$ is the input, and $\bz^\star$ is the equilibrium solved by fixed-point solvers. While taking $\mathrm{O}(1)$ memory cost because of the single layer, DEQ models are validated to attain competitive performance compared with state-of-the-art traditional deep networks in different applications \cite{deq-gnn,impflow,deq-implicit2,deq-optical-flow,deq-diffusion}.

\begin{figure}[t]
    \centering
    \includegraphics[width=0.45\textwidth]{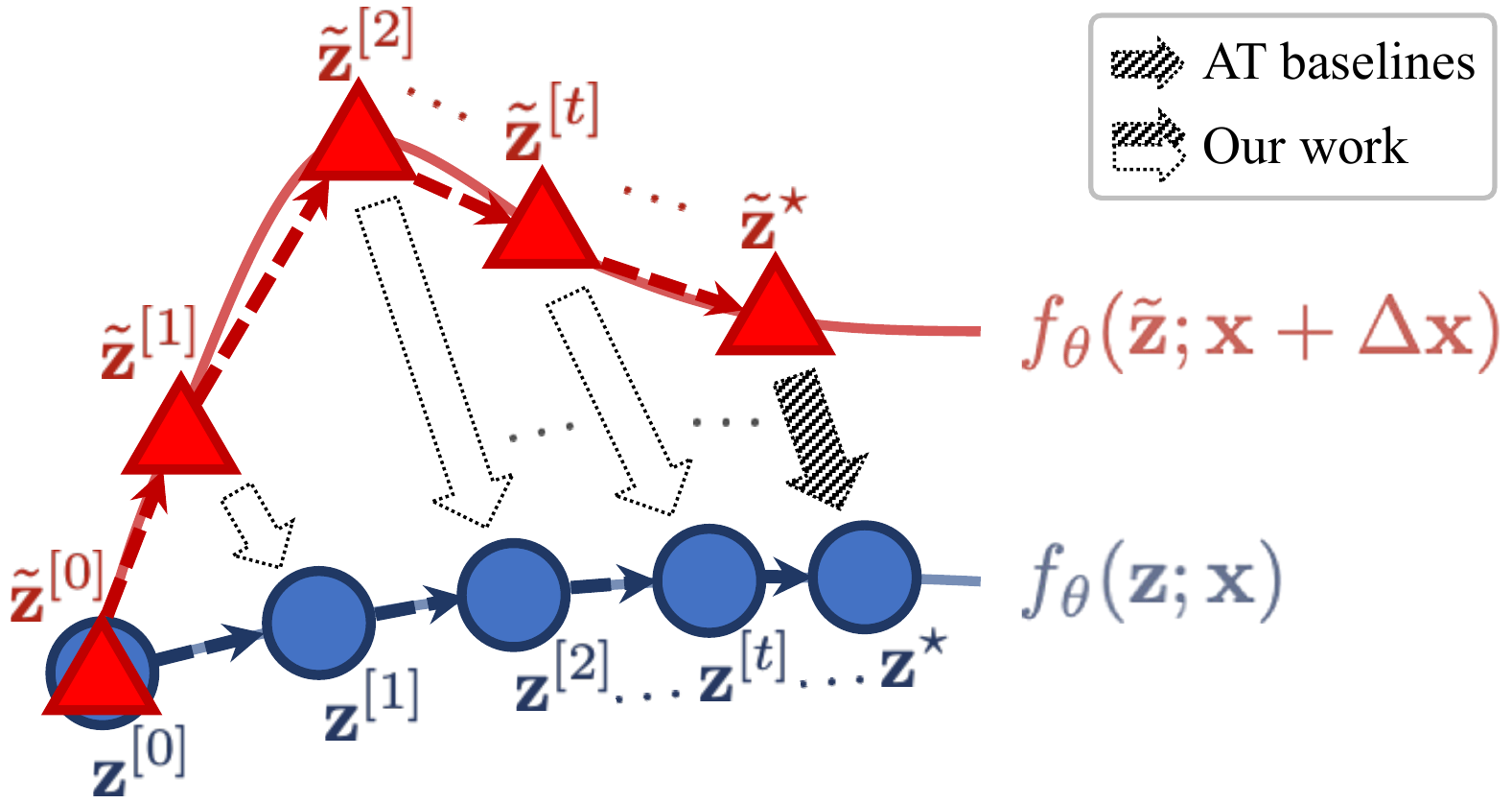}
    \vspace{-2pt}
    \caption{A \textit{conceptual} illustration of the neural dynamics in DEQ models and their regulations. Given input $\bx$, the neural dynamics are composed of intermediate states $\{\textcolor{myblue}{\bz^{[t]}}\}$ at each forward fixed-point iteration in DEQ models. With input perturbation $\Delta \bx$, the neural dynamics are shifted to $\{\textcolor{myred}{\tilde{\bz}^{[t]}}\}$ accordingly. 
    \citet{yang2022acloser} train robust general DEQs with adversarial training (AT), which imposes regulation only on the equilibrium state by enforcing $\textcolor{myred}{\tilde{\bz}^{\star}}$ to give similar predictions as $\textcolor{myblue}{\bz^{\star}}$ (the shaded arrow), leaving other intermediate states (\textit{e.g.}, $\textcolor{myred}{\tilde{\bz}^{[t]}}$) along the neural dynamics under-regulated (the hollow arrows). In this work, we propose to impose explicit regulations along the entire DEQ neural dynamics.}
    \label{fig:fig-1}
\vspace{-3pt}
\end{figure}


Promising in real-world practice, DEQ models necessitate adversarial robustness for their reliable deployment, which however remains underexplored. Most existing works that study robust DEQ models are dedicated to certifying the robustness \cite{certiDEQ-1,certiDEQ-3,certiDEQ-2,certiDEQ-4,certiDEQ-5,wei2022certified} of monotone DEQ. Monotone DEQ \cite{mondeq} is a type of DEQ model that enjoys equilibrium convergence guarantees but requires sophisticated layer and weight parameterization. In addition, the scalability of robustness certification methods also limits the scope of these prior arts for practical use. In contrast, \citet{deq-input-opt} focus on general DEQs and study their empirical adversarial robustness by accelerating the gradient-based attacks. \citet{yang2022acloser} propose white-box robustness evaluation protocols for general DEQs and conduct fair comparisons between DEQs and traditional deep networks under the adversarial training (AT) framework \cite{madry2018towards}. As reported in \cite{yang2022acloser}, however, the white-box robustness performance of general DEQs still falls behind their deep network counterparts. 
As adversarial training is a general technique that can be applied to all kinds of differentiable neural architectures, we ask the following question: 
\textit{Is it possible to exploit the structural uniquenesses of DEQ models to further improve their adversarial robustness?}

Fortunately, the \textit{neural dynamics} perspective for DEQ models brings insights into the problem. The neural dynamics perspective interprets the evolution of intermediate states in a neural model as a dynamical system \cite{eproposal,PMP}. This perspective is naturally suitable for DEQs, as their structure in Eq. (\ref{eq-1}) explicitly formulates the neural dynamics. 
From this perspective, robust neural models correspond to neural dynamics without a drastic shift in the terminal state given a perturbed initial state \cite{RobustNeuralODE,kang2021stable}, and AT enforces the \textit{terminal state} of neural dynamics to give similar predictions whether the input is clean or perturbed \cite{zhang2019theoretically}. 
Shown in Figure \ref{fig:fig-1}, for DEQ models, AT does not \textit{explicitly} regulate \textit{intermediate states} along neural dynamics. However, \citet{yang2022acloser} shows that even a DEQ model is trained by AT, its intermediate states can still be attacked, leading to poor robustness performance. This finding implies the structural specialty of DEQs differentiating from deep networks, and paves the way for explicit regulations along the neural dynamics to improve their adversarial robustness.





In this work, we exploit the structural properties of DEQs to explicitly regulate their neural dynamics for improved robustness. Drawing inspiration from the entropy in dynamical systems and its implications on system stability and robustness, we propose to reduce prediction entropy by progressively updating the inputs along the DEQ neural dynamics. We also randomly select intermediate states along the neural dynamics for loss computation in adversarial training. In this way, our methods integrate explicit regulations along the neural dynamics of DEQ models, and boost the robustness of general DEQs: On the standard white-box robustness evaluation benchmark CIFAR-10 with perturbation range $\ell_\infty = 8/255$, our DEQs achieve significantly better performance in white-box adversarial robustness compared with the results in \citet{yang2022acloser}, and even outperform the strong deep network baselines with benchmarked adversarial robustness results in \citet{pang2020bag}. We have also conducted several ablation studies to validate the effectiveness of our proposed methods. Our code is available at \url{https://github.com/minicheshire/DEQ-Regulating-Neural-Dynamics}. 



\section{Preliminaries}
\looseness=-1 Deep equilibrium models are a class of emerging neural architecture \cite{deq,mdeq}. Of all deep networks, the closest resemblance to a DEQ model is an $M$-layer deep network with weight sharing and input injection. The forward propagation process of such a deep network would be
\begin{equation}
    \bz^{[m+1]} = f_{\theta}(\bz^{[m]}; \bx), ~\bz^{[0]}=\mathbf{0},
\label{eq-2-deep-network}
\end{equation}
where $\bx \in \mathbb{R}^{l}$ is the input, and $\bz^{[m]} \in \mathbb{R}^{d}$ is the intermediate state after the $m$-th layer with $m=0, \cdots, M-1$. $f_{\theta}: \mathbb{R}^{d \times l} \rightarrow \mathbb{R}^{d}$ forms the transformation at each layer, and $\theta$ is the weight shared across different layers of the deep network. When implementing this network in an automatic differentiation engine (\textit{e.g.}, PyTorch \cite{paszke2019pytorch}), the $f_{\theta}(\cdot; \bx)$ transformation is sequentially compounded for $M$ times, and all the intermediate states $\bz^{[1]} \sim \bz^{[M]}$ need to be stored. DEQ models seek the limit of Eq. (\ref{eq-2-deep-network}) when the number of layers goes infinity: Assuming the convergence of the process, as $m \rightarrow \infty$, the state $\bz^{[m]}$ converges to the equilibrium $\bz^\star$ with $\bz^\star = f_{\theta}(\bz^\star; \bx)$, as stated in Eq. (\ref{eq-1}). 

DEQ models cast the ``infinite'' forward process of Eq. (\ref{eq-2-deep-network}) as a fixed-point iteration process to solve for the equilibrium $\bz^\star$ in Eq. (\ref{eq-1}). While the most straightforward way to do this is exactly Eq. (\ref{eq-2-deep-network}), in DEQs, advanced fixed-point solvers (\textit{e.g.}, Broyden's method \cite{Broyden}) are used to accelerate the iteration convergence. For simplicity, we abusively refer to $\bz$'s as the intermediate states of DEQs in the following and discard Eq. (\ref{eq-2-deep-network}). After $N$ iterations of the forward solver, the $\bz^{[N]}$ is returned and is numerically treated as $\bz^{\star} = \bz^{[N]}$.

Now we provide the formal definition for the neural dynamics of DEQ models, which are at the heart of our study. Neural dynamics reflect the evolution of the intermediate states in a neural model. For DEQ models, the neural dynamics consist of the sequence $\{\bz^{[1]}, \cdots, \bz^{[N]}\}$, which satisfies
\begin{equation}
    \bz^{[t+1]} = \mathrm{Solve}\!\left( \bz = f_{\theta}(\bz; \bx);~\bz^{[\leq t]} \right)
\label{eq-4-neural-dynamics}
\end{equation}
for $0\le t<N$ and $\bz^{[0]}=\mathbf{0}$. $\mathrm{Solve}$ is the fixed-point solver for the forward process in DEQs, which is usually instantiated with the Broyden's method \cite{Broyden}. At iteration $t$, the solver $\mathrm{Solve}$ bases on $\bz^{[t]}$ to compute the next intermediate state $\bz^{[t+1]}$ for the fixed-point equation. While $\{\bz^{[t]}\}$ are not stored in memory, \citet{yang2022acloser} demonstrate that the intermediate states $\bz^{[t]} \ne \bz^{\star}$ exhibit higher robustness than the equilibrium state $\bz^{\star}$, and attacks can be constructed for the intermediate $\bz^{[t]}$s. In our work, we explicitly regulate the behavior of all $\{\bz^{[t]}\}$ along the neural dynamics in DEQ models to improve their robustness.

\section{Methodology}
In this section, we demonstrate our approaches that facilitate explicit regulations along the neural dynamics of DEQs to improve robustness. We start with a short overview with two observations about the structural properties of DEQs in Sec. \ref{sec3:1-overview}. We then exploit the two uniquenesses of DEQs and propose two regulation methods in Secs. \ref{sec3:2-testing-entropy-reduction} and \ref{sec3:3-training-random}.

\subsection{Overview} \label{sec3:1-overview}

\looseness=-1 Suppose a trained DEQ model on an image classification task. Its weight $\theta$ is fixed, and the forward iteration number $N$ is constant. From Eq. (\ref{eq-4-neural-dynamics}), it can be seen that the neural dynamics are fully decided by the input $\bx$. How does the perturbation to $\bx$ affect the neural dynamics in the DEQ model? 

Assume that a clean input $\bx$ induces $\{\bz^{[t]}\}$, and a perturbed input $\bx + \Delta \bx$ induces $\{\Tilde{\bz}^{[t]}\}$. To get an intuitive understanding of the difference between them, we replace the $\mathrm{Solve}$ in Eq. (\ref{eq-4-neural-dynamics}) with the most straightforward unrolling for all $t=1, \cdots, N$. For intermediate step $t$, we have
\begin{equation}
    \bz^{[t+1]} = f_{\theta}(\bz^{[t]}; \bx),~\Tilde{\bz}^{[t+1]} = f_{\theta}(\Tilde{\bz}^{[t]}; \bx + \Delta \bx).
\end{equation}
The difference between $\Tilde{\bz}^{[t+1]}$ and $\bz^{[t+1]}$ reads
\begin{equation}
\resizebox{.99999\hsize}{!}{
    $\begin{aligned}
        & \! \| \Tilde{\bz}^{[t+1]} - \bz^{[t+1]} \| = \| f_{\theta}(\Tilde{\bz}^{[t]}; \bx + \Delta \bx) - f_{\theta}(\bz^{[t]}; \bx) \| \\
    = \; & \! \| f_{\theta}(\Tilde{\bz}^{[t]};\bx \!+\! \Delta \bx)\!-\!f_{\theta}(\Tilde{\bz}^{[t]};\bx) \;+\; f_{\theta}(\Tilde{\bz}^{[t]};\bx) \!-\! f_{\theta}(\bz^{[t]};\bx) \| \\     
    \le \; & \!\! \underbrace{\|f_{\theta}(\Tilde{\bz}^{[t]};\bx \!+\! \Delta \bx) \!-\! f_{\theta}(\Tilde{\bz}^{[t]};\bx)\|}_{\text{Perturbation from} ~\bx}\!+\! \underbrace{\|f_{\theta}(\Tilde{\bz}^{[t]};\bx) \!-\! f_{\theta}(\bz^{[t]};\bx)\|}_{\text{Accumulation in} ~\bz}.
    \end{aligned}$
}
\label{eq-difference-in-z}
\end{equation}
According to Eq. (\ref{eq-difference-in-z}), the difference between $\Tilde{\bz}^{[t+1]}$ and $\bz^{[t+1]}$ is inherited from those between $\Tilde{\bz}^{[t]}$ and $\bz^{[t]}$, and further amplified by the perturbed input $\bx + \Delta \bx$.
Fortunately, DEQs differentiate from traditional deep networks \cite{he2016deep,zagoruyko2016wide} in two structural uniquenesses:
(i) The input $\bx$ is involved in each iteration along the neural dynamics of DEQs. In contrast, conventional deep residual networks do not follow a layer-wise input-injection design. (ii) All of the intermediate states along the neural dynamics can be seamlessly sent into the classification head of the DEQ model for predictions. By comparison, for traditional deep networks like ResNets, the intermediate representations are often of different shapes from the input of the top classification layer. According to the two properties in DEQ models, we propose two methods for neural dynamics regulation in Secs. \ref{sec3:2-testing-entropy-reduction} and \ref{sec3:3-training-random}.

\begin{figure}[t]
    \centering
    \includegraphics[width=0.48\textwidth]{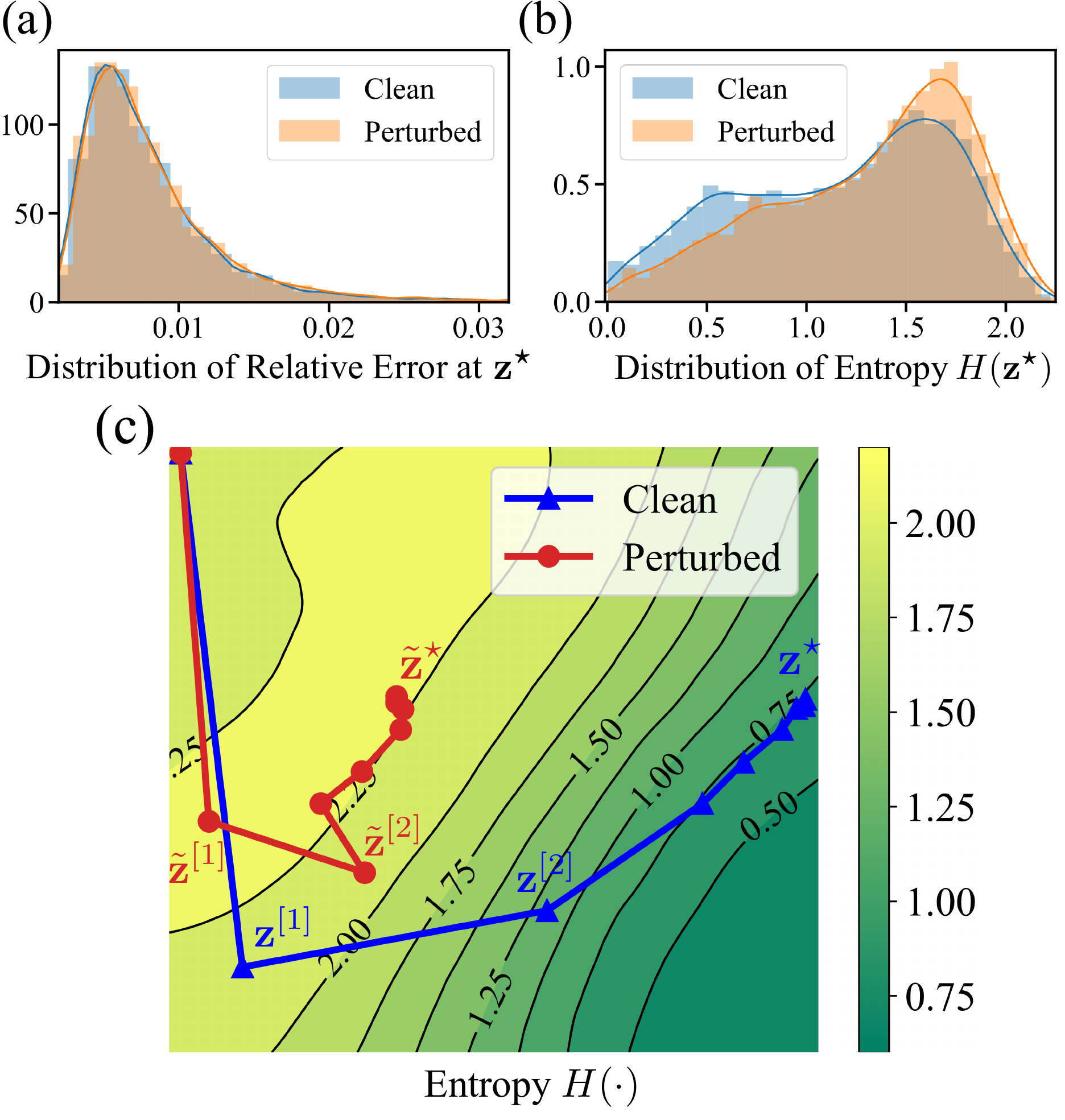}
    \vspace{-20pt}
    \caption{Observations on the prediction entropy of a DEQ model trained with AT under clean and perturbed inputs. \textbf{(a)} The distribution of relative error at the equilibrium state $\bz^{\star}$ in the forward solver. Both clean and perturbed inputs lead to converged fixed-point neural dynamics with similar relative error distribution over the validation set. \textbf{(b)} The distribution of prediction entropy at the equilibrium state. Compared with (a), the clean inputs tend to induce equilibrium states that give smaller prediction entropy than the perturbed input counterparts. \textbf{(c)} An exampled visualization of the prediction entropy of each intermediate state along the neural dynamics. Compared to the clean input with $\textcolor{blue}{{\bz}^{[t]}}$s of diminishing prediction entropy and correct predictions, the perturbed input mistakes the prediction and results in $\textcolor{red}{\Tilde{\bz}^{[t]}}$s of high prediction entropy.} 
    \label{fig:entropy-illustration}
    \vspace{-10pt}
\end{figure}

\subsection{Input Entropy Reduction Along Neural Dynamics} \label{sec3:2-testing-entropy-reduction}

\subsubsection{Observations of Prediction Entropy} \label{sec3:2:1-observations}


As noted in Sec. \ref{sec3:1-overview}, the input $\bx$ is applied along the neural dynamics in DEQ models, and $\bx$ can be either clean or perturbed. A well-trained DEQ model exhibits neural dynamics that always obey a converged fixed-point iteration process with any inputs. Shown in Fig. \ref{fig:entropy-illustration}-(a), for a DEQ model trained with AT, the relative error at the equilibrium state $\|f_\theta(\bz^{[N]}; \bx) - \bz^{[N]}\|_2 / \| f_\theta(\bz^{[N]}; \bx) \|_2$ follows similar distributions with either clean or perturbed input. 

An adversarial example $\bx+\Delta\bx$ leads to the incorrect prediction of $\Tilde{\bz}^{[N]}$. However, the initial state $\bz^{[0]}$ in DEQ models is always set to $\mathbf{0}$. This implies that with a perturbed $\bx+\Delta\bx$, the neural dynamics $\{\Tilde{\bz}^{[t]}\}$ diverge from the $\{{\bz}^{[t]}\}$ under the clean $\bx$ in predictions. While the DEQ model is determined in the fixed-point convergence of the neural dynamics, we investigate whether the intermediate states along the neural dynamics are ``determined'' in their predictions.


To characterize the ``determination'' of predictions, we adopt entropy as the measurement. On the one hand, a higher prediction entropy indicates a more flat probability distribution, with smaller probability differences among different classes. On the other hand, the theory of entropy in dynamical systems \cite{young2003entropy} shows that a dynamical system with higher entropy indicates larger Lyapunov exponents, therefore more inclined to be unstable (see a brief discussion about this in Appendix \ref{app:entropy-dynamical-systems}). Inspired by this, we compute the prediction entropy of an intermediate state $\bz^{[t]}$. Formally, the prediction entropy of $\bz^{[t]}$ is defined as
\begin{equation}
\setlength\abovedisplayskip{3pt}
\setlength\belowdisplayskip{3pt}
\scalemath{0.95}{
   H(\bz^{[t]}) = -\sum_{j=1}^{C} p_j^{[t]} \log p_j^{[t]},}
\label{eq:prediction-entropy}
\end{equation}
\looseness=-1 where $\mathbf{p}^{[t]} = h_{\Phi}\big(\bz^{[t]}\big)$ is the prediction logits vector with $\mathbf{p}^{[t]} = \left[p_j^{[t]}\right] \in \mathbb{R}^{C}$. $h_{\Phi}: \mathbb{R}^{d} \to \mathbb{R}^{C}$ is the classification head in the DEQ model with parameters $\Phi$ and $C$ classes.

We use Eq. (\ref{eq:prediction-entropy}) to investigate the prediction entropy of the equilibrium state in DEQ models with clean or perturbed inputs. Surprisingly, shown in Fig. \ref{fig:entropy-illustration}-(b), we find that the perturbed inputs result in the equilibrium states with higher prediction entropy than the clean inputs from a distributional perspective. It is inferred that the neural dynamics with high prediction entropy are prone to give incorrect predictions. 

\looseness=-1 We further visualize an example of the prediction entropy along all the intermediate states of the neural dynamics in Fig. \ref{fig:entropy-illustration}-(c), with clean $\bx$ or perturbed $\bx+\Delta\bx$ as the input. With the clean input $\bx$, the prediction entropy diminishes along the corresponding neural dynamics. In contrast, with the perturbed input $\bx+\Delta\bx$, the prediction entropy for each intermediate state remains high, and the resulting prediction is mistaken. To guide the perturbed dynamics toward the clean one, we propose to \textit{reduce the prediction entropy by progressively updating the input along the neural dynamics}. In the next section, we provide its optimization framework.

\subsubsection{Input Entropy Reduction Framework}

Given a potentially perturbed input $\bx$, the entropy reduction framework with progressive input updates is formalized as
\begin{equation}
\vspace{-5pt}
\begin{aligned}
    \min_{\bu^{[1]},\cdots,\bu^{[N]}} ~ &H(\bz^{[N]}), \\
    {\rm s.t.} \quad\;\;\,\, &\bz^{[t+1]} = \mathrm{Solve}\!\left( \bz = f_{\theta}(\bz; \bx+\bu^{[t]});~\bz^{[\leq t]} \right), \\
    &\bu^{[t]} \in [-\epsilon, \epsilon]^{l},
\end{aligned}
\label{prob:input-ent-reduction}
\end{equation}
\looseness=-1 where $t=1,\cdots,N$, and $\{\bu^{[t]}\}$ are the updates on the input $\bx$ with range constraints. From (\ref{prob:input-ent-reduction}), the $\{\bu^{[t]}\}$ can be viewed as the controllers along the neural dynamics. Framework~(\ref{prob:input-ent-reduction}) then forms an optimal control problem with the aim of guiding the neural dynamics towards reduced entropy at the final state $\bz^{[N]}$. As demonstrated in \citet{PMP,lu-adv}, by solving the problem with Pontryagin Maximum Principle (PMP) \cite{ode-adjoint}, the gradient descent methods are derived to obtain the optimal controllers $\{\bu^{[t]}\}$.
We employ the iterative projected gradient descent framework to optimize for $\bu^{[t]}$. Specifically, at iteration $t$, after $\bz^{[t+1]}$ in the neural dynamics is obtained by Eq. (\ref{eq-4-neural-dynamics}) given $\bz^{[\le t]}$ and $\bx^{[t]}$, the input $\bx^{[t]}$ is updated to reduce prediction entropy at $\bz^{[t+1]}$ for $R$ iterations:
\begin{equation}
\vspace{-5pt}
    \!\!\!\!\bx^{[t]}_{i} \!=\! {\rm Proj}_{[-\epsilon, \epsilon]^{l}}\!\left( \bx^{[t]}_{i-1} \!-\! \beta \nabla_{\bx} H(f_{\theta}(\bz^{[t+1]}; \bx^{[t]}_{i-1})) \right),\!
\label{eq:input-update-entropy-reduction}
\end{equation}
with $\beta$ as the step size, $i=1, \cdots, R$, $\bx_{0}^{[t]}=\bx^{[t]}$, and $\bx^{[t+1]}=\bx^{[t]}_R$ forms $\bx+\bu^{[t+1]}$. After the updates of the input, the state $\bz^{[t+1]}$ is re-calculated with $\bx^{[t+1]}$ in Eq.~(\ref{eq-4-neural-dynamics}). 
Solving for $\bu^{[t]}$ for each $t$ can be time-consuming. In the implementation, we intervene to optimize for $\bu^{[t]}$ every $T_f$ states along the neural dynamics, \textit{i.e.}, requiring $\bu^{[pT_f+q]}=\bu^{[pT_f+1]}$,~$\forall p,q\in \mathbb{N},
~p\ge0, ~1<q\le T_f$. 


\begin{algorithm}[t]
    \renewcommand{\algorithmicrequire}{\textbf{Input:}}
    \renewcommand{\algorithmicensure}{\textbf{Output:}}
    \caption{\textsc{Entropy Reduction in Sec. \ref{sec3:2-testing-entropy-reduction}}}
    \label{alg:entropy_reduction}
    \begin{algorithmic}[1]
        \REQUIRE Input $\bx$, iteration number $R$ and frequency $T_f$
        \ENSURE  The equilibrium state $\bz^{\star}=\bz^{[N]}$
        \STATE \textcolor{blue}{ /* Test Phase */}
        \STATE $\bz^{[0]}=\mathbf{0}$, $\bx^{[0]}=\bx$
        \FOR{$t=0, \cdots, N-1$}
        \STATE    $\bz^{[t+1]} = \mathrm{Solve}\!\left( \bz = f_{\theta}(\bz; \bx^{[t]});~\bz^{[\leq t]} \right)$
        \STATE    $\bx^{[t+1]}=\bx^{[t]}$
            \IF{$~(t+1) ~\mathrm{\mathbf{mod}}~ T_f = 0~$}
                \STATE \textcolor{blue}{ /* Reducing $\bx^{[t]}$ Prediction Entropy for $R$ Iters */} 
                \STATE $\bx^{[t]}_0=\bx^{[t]}$
                \FOR{$i=1, \cdots, R$}
                    \STATE Update $\bx^{[t]}_{i-1}$ into $\bx^{[t]}_{i}$ with Eq. (\ref{eq:input-update-entropy-reduction})
                \ENDFOR
                \STATE $\bx^{[t+1]}=\bx^{[t]}_{R}$
                \STATE $\bz^{[t+1]} = \mathrm{Solve}\!\left( \bz = f_{\theta}(\bz; \bx^{[t+1]});~\bz^{[\leq t]} \right)$
            \ENDIF
        \ENDFOR
    \end{algorithmic}
\end{algorithm}

The whole process of the input entropy reduction is shown in Algo. \ref{alg:entropy_reduction}. In this way, the original neural dynamics are mounted to a regulated ``orbit'' from $\bz^{[t]}$, which would crucially impact the states afterward and the predictions they give, and eventually result in reduced $H(\bz^{[N]})$. 


Our method is also related to the joint optimization for inputs and states in DEQ models \cite{deq-input-opt}. However, in this work, we do not couple the update of $\bz^{[t+1]}$ with Eq. (\ref{eq:input-update-entropy-reduction}) and refrain from the calculation of the joint Jacobian for $\bx$ and $\bz$. We leave more efficient and effective regulations for neural dynamics in DEQs as future work.


\begin{table*}[t]
    \centering
    \caption{Comparisons among the robustness performance (\%) of deep networks (\textsc{ResNet-18}) \cite{pang2020bag} and the DEQ models of similar parameter counts (\textsc{DEQ-Large}) with our methods in different AT frameworks on CIFAR-10 test set. We leverage the early-state defense proposed in \citet{yang2022acloser} for all DEQ models, and use ``\textsc{All}'' to represent the minimum of the \textsc{PGD} attack and AutoAttack (``\textsc{AA}''). Our methods significantly improve the robustness of DEQs over \cite{yang2022acloser}, and even outperform the strong deep network baselines with the TRADES framework. $^\ddagger$ indicates our implementation; $^{*}$ denotes that the results are brought from \citet{pang2022robustness}.} 
    \resizebox{0.85\textwidth}{!}{\begin{tabular}{lllcccl}
    \toprule
         \textsc{AT Framework} & \textsc{Architecture} & \textsc{Method} & \textsc{Clean} & \textsc{PGD} & \textsc{AA} & \textsc{All} \\
    \midrule
        \multirow{6}{*}{\textsc{PGD-AT}} & \textsc{ResNet-18} & \textsc{\citet{pang2020bag}} & \bf 82.52 & 53.58 & 48.51 & 48.51 \\ 
    \cmidrule{2-7}
         ~ & \multirow{5}{*}{\textsc{DEQ-Large}} & \textsc{\citet{yang2022acloser}} & 79.67 & 47.12 & 48.37 & 47.12 \\
        ~ & ~ & \textsc{\citet{yang2022acloser}$^{\ddagger}$} & 77.89 & 49.45 & 47.58 & 47.58 \\
        ~ & ~ & ~\textsc{+ Ours (Sec. ~\,\ref{sec3:2-testing-entropy-reduction})} & 77.51 & \bf 51.62 & \bf 49.31 & \bf 49.31 \\
        ~ & ~ & ~\textsc{+ Ours (Sec. ~\,\ref{sec3:3-training-random})} & 78.93 & 48.18 & 48.09 & 48.09 \\
        ~ & ~ & ~\textsc{+ Ours (Secs. \ref{sec3:2-testing-entropy-reduction} \& \ref{sec3:3-training-random})} & 80.63 & 49.22 & 43.79 & 43.79 \\
    \midrule
        \multirow{5}{*}{\textsc{TRADES}} & \textsc{ResNet-18} & \textsc{\citet{pang2020bag}$^{*}$} & \bf 81.47 & - & 49.14 & 49.14 \\ 
    \cmidrule{2-7}
         ~ & \multirow{4}{*}{\textsc{DEQ-Large}} & \textsc{\citet{yang2022acloser}}$^{\ddagger}$ & 74.92 & 50.46 & 50.33 & 50.33 \\
        ~ & ~ & ~\textsc{+ Ours (Sec. ~\,\ref{sec3:2-testing-entropy-reduction})} & 73.80 & 51.41 & 50.52 & 50.52 \\
        ~ & ~ & ~\textsc{+ Ours (Sec. ~\,\ref{sec3:3-training-random})} & 77.64 & 51.10 & 49.64 & 49.64 \\
        ~ & ~ & ~\textsc{+ Ours (Secs. \ref{sec3:2-testing-entropy-reduction} \& \ref{sec3:3-training-random})} & 78.89 & \bf 55.18 & \bf 51.50 & \bf 51.50  \\        
    \bottomrule         
    \end{tabular}}
    \label{tab:final-comparisons}
\end{table*}

\subsection{Loss from Random Intermediate States} \label{sec3:3-training-random}

\looseness=-1 In addition to progressively updating the input during testing, we propose another technique for the explicit regulation of the neural dynamics in DEQ models. As shown in Eq. (\ref{eq-difference-in-z}), the second term reflects the difference accumulated in $\Tilde{\bz}^{[t]}$ from ${\bz}^{[t]}$. A straightforward approach to imposing explicit regulations on the intermediate state $\Tilde{\bz}^{[t]}$ is to calculate the adversarial loss using \textit{random intermediate states} during AT. 

Formally, for the vanilla AT baselines, the loss function $L$ in the objective is calculated using only the equilibrium state:
\begin{equation}
    \min_{\theta,\Phi} \max_{\Delta \bx \in [-\epsilon, \epsilon]^{l}} L\left(h_\Phi \big(\Tilde{\bz}^{[N]} \big), y\right),
\label{original-AT-loss}
\end{equation}
with the equilibrium state $\Tilde{\bz}^{[N]}$ satisfying Eq. (\ref{eq-4-neural-dynamics}) with $\bx+\Delta \bx$ as the input, and $y$ is the ground-truth label for $\bx$. Our method calculates Eq. (\ref{original-AT-loss}) with random intermediates:
\begin{equation}
    \min_{\theta,\Phi} \max_{\Delta \bx \in [-\epsilon, \epsilon]^{l}} \mathbb{E}_{i \in \mathcal{U}[1,N]}~ L\left(h_\Phi \big(\Tilde{\bz}^{[i]} \big), y\right),
\label{eq:sec33-our-formulation}
\end{equation}
where we randomly select intermediate states $\Tilde{\bz}^{[i]}$ inside the forward fixed-point solver for loss computation. In this way, all the intermediates are imposed with explicit regulations without violating the $O(1)$ memory constraint of DEQ models. We thus expect their neural dynamics to be less deviated under attacks and exhibit higher robustness. 

\section{Experiments}\label{sec:experiments}
\textbf{Setup.}
We follow the settings in \citet{yang2022acloser} of the configurations of DEQ model architecture: the large-sized DEQ with its parameter count similar to ResNet-18. The number of iterations $N$ in the forward solver is $8$. 
For adversarial training frameworks, we use both PGD-AT \cite{madry2018towards} and TRADES \cite{zhang2019theoretically}. PGD-AT is used in the previous study on robust DEQ models \cite{madry2018towards}, while the regularization term for robustness in TRADES shares similarity with Eq. (\ref{eq-difference-in-z}). We experiment on CIFAR-10 \cite{Krizhevsky2012} with $\ell_\infty$ perturbation range $\epsilon=8/255$. The default hyperparameter setting for the Sec. \ref{sec3:2-testing-entropy-reduction} method is $\beta=2/255$, $R=10$, and $T_f=2$. The detailed settings are listed in Appendix \ref{app:exp-details}.

We follow to use the white-box robustness evaluation protocol proposed in \citet{yang2022acloser}: We use the early-state defense by selecting the intermediate state with the highest accuracy under ready-made PGD-10 to compute for predictions. As the intermediate state is non-differentiable, we adopt the proposed intermediate unrolling method that estimates the gradients used to attack the state. Specifically, the gradients used in the attacks are calculated by unrolling an intermediate state $\bz^{[i]}$ for $K_a$ steps:
\begin{equation}
    \bz_a^{[i+j]} = (1-\lambda)\bz_a^{[i+j-1]} + \lambda f_\theta (\bz_a^{[i+j-1]}; \bx),
\label{eq:unrolled-intermediates}
\end{equation}
with $\bz_a^{[i]} = \bz^{[i]}$ and $j=1,\cdots,K_a$, and $\bz_a^{[i+K_a]}$ is used to compute the loss and take the gradient. In our work, we provide a systematic evaluation by covering all $1 \le i \le N=8$, $1 \le K_a \le 9$, and $\lambda \in \{0.5, 1\}$. Unless specified, all of the single robustness performance that is reported ``under intermediate attacks'' is the minimum accuracy over $8\times9\times2=144$ attacks in the form of Eq. (\ref{eq:unrolled-intermediates}).


\subsection{Main Results}

\looseness=-1 Table \ref{tab:final-comparisons} shows the robustness comparisons among traditional deep networks ResNet-18 and the DEQ models with a similar amount of parameters (DEQ-Large). For the DEQ models, we use the original adversarial training framework and compose it with our methods. According to the results, Both the test-time input entropy reduction in Sec. \ref{sec3:2-testing-entropy-reduction} and the training-time loss computation with random intermediates in Sec. \ref{sec3:3-training-random} improve DEQ model robustness over the vanilla AT baselines. Using our methods, the robustness performance of the DEQ-Large models significantly out-performs the DEQ-Large baselines in \citet{yang2022acloser}, and even surpasses that of ResNet-18 \cite{pang2020bag}.

\begin{table}[t]
    \centering
    \caption{Comparisons between the ready-made PGD-10 at the equilibrium state $\bz^{[N]}$ and the intermediate attack at state $\bz^{[i]}$ (Eq. (\ref{eq:unrolled-intermediates})). The intermediate attacks are always more effective, as they result in lower accuracy (\%) of the model than the ready-made ones.}
    \resizebox{0.48\textwidth}{!}{\begin{tabular}{lllll}
    \toprule
        ~ & ~ & \multicolumn{3}{c}{\textsc{PGD At Which} $\bz^{[t]}$} \\
        \textsc{AT} & \textsc{Method} & \textsc{Final} & \textsc{Inter.} & \textsc{Diff.} \\
    \midrule
        \multirow{4}{*}{\textsc{PGD-AT}} & \textsc{\citet{yang2022acloser}} & 50.55 & 49.45 & 1.10 \\
        ~ & ~\textsc{+ Sec. ~\,\ref{sec3:2-testing-entropy-reduction}} & 53.01 & \bf 51.62 & 1.39 \\
        ~ & ~\textsc{+ Sec. ~\,\ref{sec3:3-training-random}} & 51.05 & 48.18 & 2.87 \\
        ~ & ~\textsc{+ Secs. \ref{sec3:2-testing-entropy-reduction} \& \ref{sec3:3-training-random}} & 54.91 & 49.22 & 5.69 \\
    \midrule
        \multirow{4}{*}{\textsc{TRADES}} & \textsc{\citet{yang2022acloser}} & 51.92 & 50.46 & 1.46 \\
        ~ & ~\textsc{+ Sec. ~\,\ref{sec3:2-testing-entropy-reduction}} & 53.74 & 51.41 & 2.33 \\
        ~ & ~\textsc{+ Sec. ~\,\ref{sec3:3-training-random}} & 52.67 & 51.10 & 1.57 \\
        ~ & ~\textsc{+ Secs. \ref{sec3:2-testing-entropy-reduction} \& \ref{sec3:3-training-random}} & 56.09 & \bf 55.18 & 0.91 \\    
    \bottomrule         
    \end{tabular}}
    \label{tab:diff-between-interm-final-attacks}
    \vspace{-5pt}
\end{table}

\textbf{Intermediate attacks are strong.} As we conduct comprehensive intermediate-state attack experiments, the validated robustness is more reliable than only using ready-made attacks. Table \ref{tab:diff-between-interm-final-attacks} demonstrates that the intermediate-state PGD attacks always result in a larger decrease of white-box adversarial robustness than off-the-shelf attacks at the final state. According to Table \ref{tab:final-comparisons}, the effect of AutoAttack \cite{croce2020reliable} is usually stronger than the PGD-10 attacks in TRADES experiments. This is opposite to the performance reported in \citet{yang2022acloser}, as they argue that the AutoAttack will overfit to the inaccurate gradient estimations by unrolled intermediates and result in attacks weaker than PGD-10. However, in our work, we circumvent the overly inaccurate gradient estimations by scrutinizing all possible pairs of $(i, K_a)$ in Eq. (\ref{eq:unrolled-intermediates}), leading to stronger adaptive-size PGD in AutoAttack. The composition of our methods yields the top robustness in the DEQ models trained with TRADES. The possible reason is that the regularization term in TRADES is more suitable for the regulations on Eq. (\ref{eq-difference-in-z}). 
In the following sections, we conduct further evaluation and analysis with the TRADES-trained DEQ models.

\subsection{Robustness Evaluation for Test-Time Defense}

Sec. \ref{sec3:2-testing-entropy-reduction} describes the algorithm for prediction entropy reduction by iteratively updating the input along the neural dynamics. As the algorithm works at inference time, we follow the guidelines in \citet{croce2022evaluating} to evaluate its robustness. Based on the intermediate-state Eq. (\ref{eq:unrolled-intermediates}) as adaptive attacks for DEQ models, we transfer the attacks across different defense methods of ours (``TRADES Baseline'', ``TRADES + Sec. \ref{sec3:2-testing-entropy-reduction}'', ``TRADES + Sec. \ref{sec3:3-training-random}'', and ``TRADES + Secs. \ref{sec3:2-testing-entropy-reduction} and \ref{sec3:3-training-random}''), and further employ the adaptive-size APGD-CE and the score-based Square \cite{andriushchenko2020square} attacks. It is noted that our Sec. \ref{sec3:2-testing-entropy-reduction} defense operates on the neural dynamics, which lie along the forward pass of DEQ models only. Therefore, following the comments from \citet{croce2022evaluating} to \citet{yoon2021adversarial}, we do not equip BPDA \cite{athalye2018obfuscated} with the attacks for our defense as it is unnecessary.

\begin{table}[t]
    \centering
    \caption{The performance (\%) of each type of our methods under the attacks transferred from other types of our methods with TRADES. We transfer all the intermediate attacks defined by Eq.~(\ref{eq:unrolled-intermediates}) and report the minimum accuracy. Given one of our methods, it is shown that the attacks transferred from other defenses are always weaker than the attacks designed for the given method itself. The lowest robustness for each defense is shaded.}
    \resizebox{0.478\textwidth}{!}{\begin{tabular}{lcccc}
        \toprule
            ~ & \multicolumn{4}{c}{\textsc{PGD Attacks Transferred From}} \\
            \textsc{Method} & \textsc{Base} & \textsc{B+S\ref{sec3:2-testing-entropy-reduction}} & \textsc{B+S\ref{sec3:3-training-random}} & \textsc{B+S\ref{sec3:2-testing-entropy-reduction}\&\ref{sec3:3-training-random}} \\
        \midrule
        \textsc{Base} & \cellcolor{lightgray}{\bf 50.46} & 50.73 & 56.66 & 57.37 \\
        \textsc{+ S\ref{sec3:2-testing-entropy-reduction}} & 51.58 & \cellcolor{lightgray}{\bf 51.41} & 55.60 & 56.09 \\
        \textsc{+ S\ref{sec3:3-training-random}} & 60.00 & 60.24 & \cellcolor{lightgray}{\bf 51.10} & 52.89 \\
        \textsc{+ S\ref{sec3:2-testing-entropy-reduction}\&\ref{sec3:3-training-random}} & 59.87 & 60.06 & 55.19 & \cellcolor{lightgray}{\bf 55.18} \\
        \bottomrule
    \end{tabular}}
    \label{tab:attack_transfer}
    \vspace{-5pt}
\end{table}

\begin{table}[t]
    \centering
    \caption{Robustness performance (\%) of our methods with Secs. \ref{sec3:2-testing-entropy-reduction}, \ref{sec3:3-training-random} and TRADES under strong adaptive attacks. We choose the top $(i,K_a)$ pairs in Eq. (\ref{eq:unrolled-intermediates}) that lead to significant accuracy decreases with \textsc{PGD}, and equip them with \textsc{APGD-CE} \cite{croce2020reliable} and \textsc{Square} \cite{andriushchenko2020square}. The results in Table \ref{tab:final-comparisons} are reported with $(i,K_a)\!=\!(3,5)$, as it forms the strongest attacks. $^{*}$ denotes evaluation with 1,000 test samples.} 
    \resizebox{0.482\textwidth}{!}{\begin{tabular}{lccccccc}
        \toprule
        \multirow{2}{*}{\textsc{Attack}} & \multicolumn{7}{c}{$(i,K_a)$ \textsc{in Eq. (\ref{eq:unrolled-intermediates})}} \\
        ~ & (3,5) & (3,4) & (3,6) & (3,7) & (4,5) & (5,5) & (6,5) \\
        \midrule
        \textsc{PGD} & \bf 55.18 & 55.36 & 55.24 & 55.35 & 55.24 & 55.53 & 55.59 \\
        \textsc{APGD} & \bf 53.29 & 56.75 & 56.95 & 57.07 & 53.38 & 53.92 & 53.91 \\ 
        \textsc{Square}$^{*}$ & 67.00 & 67.23 & 67.28 & 67.13 & \bf 66.65 & 67.00 & 66.75 \\ 
        \bottomrule
    \end{tabular}}
    \label{tab:strong-attacks-for-trades-3233}
    \vspace{-10pt}
\end{table}


\textbf{Attack transferability.} We transfer all the intermediate-state attacks among the four different defense methods based on TRADES to one another. The minimum robust accuracy for each setting is reported in Table \ref{tab:attack_transfer}. It is shown that for each defense method, its strongest attack is still constructed against the method itself. For a robustly-trained DEQ model (either by the TRADES framework or plus the Sec. \ref{sec3:3-training-random} method), the adversarial examples have similar effects whether the Sec. \ref{sec3:2-testing-entropy-reduction} method is used or not. 

\textbf{Exploiting APGD-CE and Square attacks.} We select the $(i, K_a)$ pairs in Eq. (\ref{eq:unrolled-intermediates}), which form the intermediate-state attacks that trigger severe drops in accuracy with PGD-10. Specifically, $(i,K_a)\!=\!(3,5)$ forms the strongest attacks, and the results in Table \ref{tab:final-comparisons} are reported under this setting as well. We then implement these $(i, K_a)$ settings in APGD-CE and Square attacks. Table \ref{tab:strong-attacks-for-trades-3233} shows the effect of these attacks on our strongest defense ``TRADES~+~Secs. \ref{sec3:2-testing-entropy-reduction}~\&~\ref{sec3:3-training-random}''. For Square attacks, due to the time limit, we evaluate the performance on 1,000 test samples. According to the results, the adaptive-size APGD-CE is stronger than PGD, and the defense method retains higher accuracy under Square attacks than PGD and APGD-CE attacks. These phenomena agree with the performance of deep networks \cite{croce2021robustbench}. Finally, the robustness performances among different settings are similar, indicating the robustness of our defense to the configurations in Eq. (\ref{eq:unrolled-intermediates}) as well. 



\begin{figure*}[t]
    \centering
    \includegraphics[width=\textwidth]{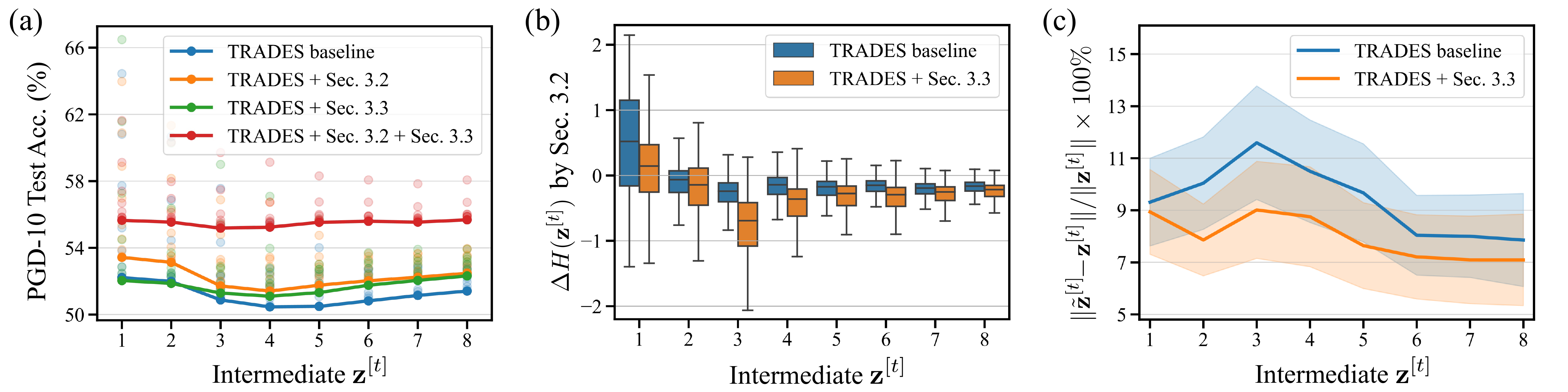}
    \vspace{-15pt}
    \caption{Comparisons among the original TRADES baseline and our methods along the neural dynamics $\bz^{[t]}$ produced by the forward solver in DEQ models. \textbf{(a)} PGD-10 attacks with varied unrolling steps at different intermediate states along the neural dynamics. At each $\bz^{[t]}$, the robustness under the strongest attack is plotted by a solid dot, with others by transparent dots. It is shown that the overall strongest attack is constructed by the state around the middle of the neural dynamics ($\bz^{[3]}$ - $\bz^{[5]}$). The composition of Secs. \ref{sec3:2-testing-entropy-reduction} and \ref{sec3:3-training-random} forms the most robust DEQ model. \textbf{(b)} The effect of entropy reduction by Sec. \ref{sec3:2-testing-entropy-reduction}. For DEQ models trained with both the TRADES baseline and the loss calculated by random intermediates (Sec. \ref{sec3:3-training-random}), the Sec. \ref{sec3:2-testing-entropy-reduction} method always manages to reduce the entropy along the neural dynamics, with the largest decrease in $H(\bz^{[t]})$ at $\bz^{[3]}$. The entropy reduction effect is more significant in the DEQ model trained with ``TRADES + Sec. \ref{sec3:3-training-random}'', which also accounts for the corresponding robustness improvement in Table \ref{tab:final-comparisons}. \textbf{(c)} The relative difference of $\Tilde{\bz}^{[t]}$ (with perturbed inputs) from $\bz^{[t]}$ (with clean inputs). When trained with the method in Sec. \ref{sec3:3-training-random}, the DEQ model shows less-deviated neural dynamics under input perturbations than the ``TRADES baseline'', thus also demonstrating better robustness in Table \ref{tab:final-comparisons}. }
    \label{fig:attack-along-neural-dynamics}
\end{figure*}


\vspace{-10pt}
\section{Analysis and Discussion}

\subsection{Robustness Improvement Along Neural Dynamics}

In this section, we conduct an in-depth analysis of the robustness improvement from our methods with TRADES by investigating the neural dynamics of the DEQ models. We use PGD-10 as it also reliably reflects model robustness in Table \ref{tab:final-comparisons}, while being faster than the adaptive-size attacks.

We first plot the PGD-10 robust accuracy for all the attacks in Eq. (\ref{eq:unrolled-intermediates}) at different intermediate states $\bz^{[t]}$ with various unrolling steps $K_a$ and $\lambda \in \{0.5, 1.0\}$. In Fig. \ref{fig:attack-along-neural-dynamics}-(a), we plot the lowest robust accuracy of the model under the attack constructed by unrolling $\bz^{[t]}$ as a solid dot. The accuracy results under other attacks with $\bz^{[t]}$ unrolling are depicted in transparency. Along the neural dynamics, it is observed that the strongest attack lies around the middle, namely, by unrolling the intermediate state of $\bz^{[3]}$ - $\bz^{[5]}$. The overall robustness performance of the model is determined by the lowest accuracy among the solid dots. Our composed method of Secs. \ref{sec3:2-testing-entropy-reduction} and \ref{sec3:3-training-random} forms the strongest defense, as the lowest accuracy it obtains is much higher than all the other methods. Opposite to \citet{yang2022acloser}, we find that $\lambda=0.5$ in Eq. (\ref{eq:unrolled-intermediates}) results in attacks always stronger than $\lambda=1.0$ (see detailed comparisons in Appendix \ref{app:additional-ablation-lambda}). 

\looseness=-1 To validate the effectiveness of the Sec. \ref{sec3:2-testing-entropy-reduction} method, we quantize the entropy reduction $\Delta H(\bz^{[t]})$ given each input from the validation set perturbed by the strongest intermediate PGD attacks. The distribution of $\Delta H(\bz^{[t]})$ for each $t$ is illustrated in Fig. \ref{fig:attack-along-neural-dynamics}-(b). As Sec. \ref{sec3:2-testing-entropy-reduction} updates the input along the neural dynamics, each $\bz^{[t]}$ is correspondingly modified, thus leading to the difference in terms of its prediction entropy. According to Fig. \ref{fig:attack-along-neural-dynamics}-(b), $\Delta H(\bz^{[t]})<0$ for each $t>1$. This means the entropy at each $\bz^{[t]}$ is reduced for $t>1$, which verifies the effectiveness of our method. Specifically, the largest deterioration of $\Delta H(\bz^{[t]})$ happens at $\bz^{[3]}$, which is also around the middle of the neural dynamics.

Finally, we demonstrate the effect of Sec. \ref{sec3:3-training-random} by plotting the relative difference of intermediate states ($\bz^{[t]}$ and $\Tilde{\bz}^{[t]}$) along the neural dynamics given a clean input $\bx$ and the perturbed one $\bx+\Delta\bx$. Shown in Fig. \ref{fig:attack-along-neural-dynamics}-(c), when trained with the loss computed with random intermediates, the DEQ model exhibits neural dynamics with the less relative difference among the clean and the perturbed inputs than the baseline. This analysis also accounts for the superiority of the Sec.~\ref{sec3:3-training-random} method in adversarial robustness, as shown in Table \ref{tab:final-comparisons}.

\begin{table}[t]
    \centering
    \vspace{-7pt}
    \caption{Robustness performance (\%) of the test-time method in Sec. \ref{sec3:2-testing-entropy-reduction} with different settings of the frequency $T_f$ and the iteration number $R$, with the underlying DEQ model trained with TRADES + the Sec. \ref{sec3:3-training-random} method. For each pair of $T_f$ and $R$, we evaluate the robustness accuracies using all PGD attacks along the neural dynamics of the DEQ model and report the minimum among them.}   
    \resizebox{0.38\textwidth}{!}{\begin{tabular}{l|cccc}
        \toprule
                & $R=1$ & $R=5$ & $R=10$ & $R=20$  \\
        \midrule
        $T_f=1$ & 54.67 & 54.96 & 55.07 & 55.03 \\
        $T_f=2$ & 54.15 & 55.05 & 55.18 & \bf 55.23 \\
        $T_f=4$ & 53.39 & 54.80 & 54.85 & 54.85 \\
        $T_f=8$ & 52.85 & 52.80 & 52.83 & 52.81 \\
        \bottomrule
    \end{tabular}}
    \label{tab:ablation-effect-sec32}
    \vspace{-7pt}
\end{table}

\vspace{-5pt}
\subsection{Effect of $T_f$ and $R$ in Sec. \ref{sec3:2-testing-entropy-reduction}}
\vspace{-3pt}

In this section, we ablate the effect of $T_f$ and $R$ in the Sec. \ref{sec3:2-testing-entropy-reduction} method. During the input entropy reduction process, $T_f$ controls the intervention frequency of input updates along the neural dynamics, and $R$ denotes the iteration number within each intervention of input entropy reduction. According to the results in Table \ref{tab:ablation-effect-sec32}, a larger $T_f$ leads to relatively lower robustness. This indicates that the Sec. \ref{sec3:2-testing-entropy-reduction} method would benefit from frequent input entropy reduction. Similarly, the method becomes more effective as $R$ increases. As the time consumption is proportional to $R$ and inversely proportional to $T_f$, we set $R=20$ and $T_f=2$ in our main experiments (see Appendix \ref{app:time-cost} for time cost analysis).


\begin{figure}[t]
    \centering
    \includegraphics[width=0.4\textwidth]{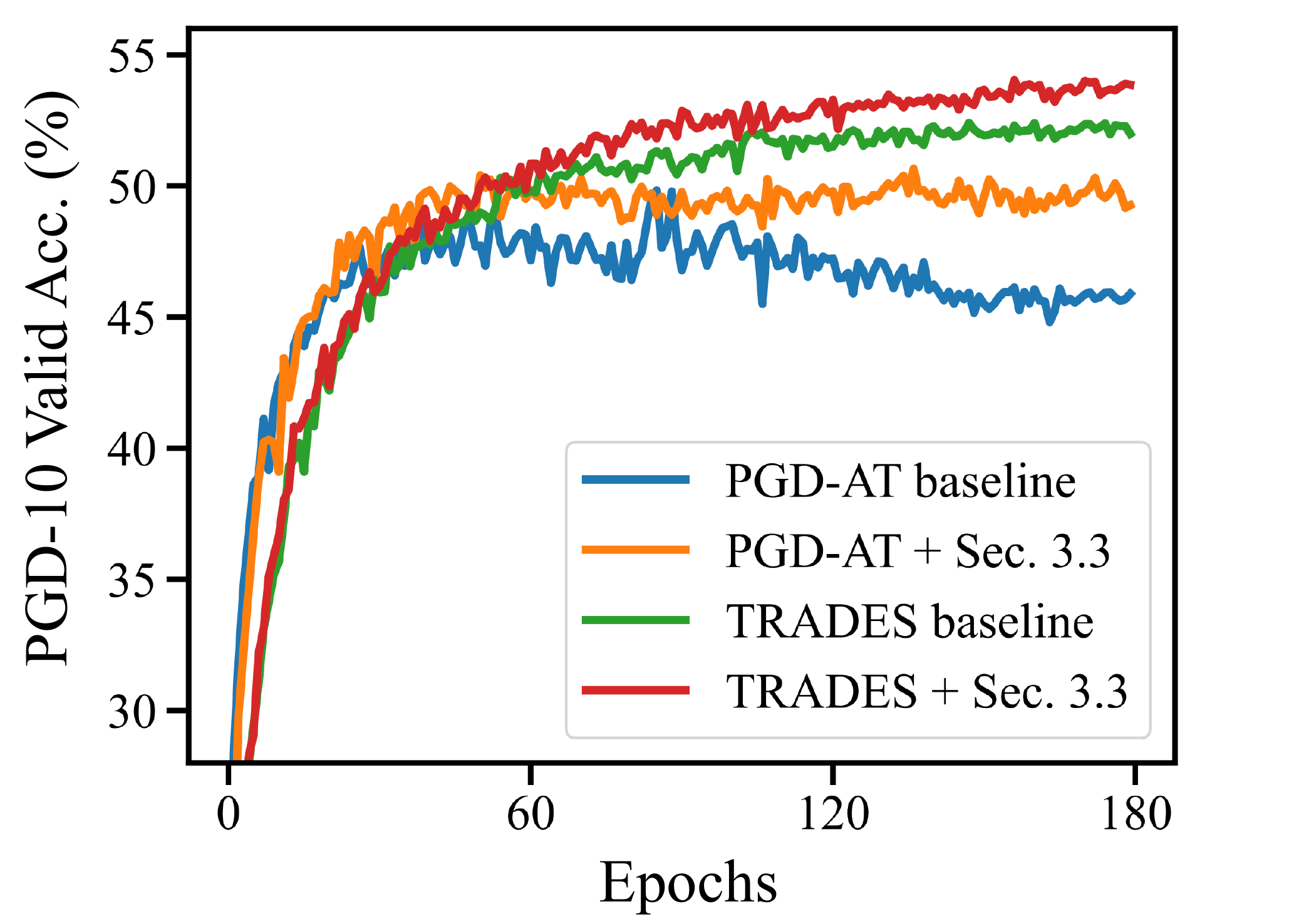}
    \vspace{-5pt}    
    \caption{Robustness performance of different methods under ready-made PGD-10 on the CIFAR-10 validation set in DEQ training. The integration of Sec. \ref{sec3:3-training-random} stabilizes the training process and yields consistent robustness improvement over the AT baselines.}
    \label{fig:training-sec33-analysis}
    \vspace{-5pt}
\end{figure}
\subsection{Effect of Sec. \ref{sec3:3-training-random} During Training}
In Fig. \ref{fig:training-sec33-analysis}, we visualize the training process of each AT baseline and its improved version with the Sec. \ref{sec3:3-training-random} method. The robust accuracy reported in Fig. \ref{fig:training-sec33-analysis} is evaluated at the final state $\bz^{\star}$. It is witnessed that the Sec. \ref{sec3:3-training-random} method always improves over the AT baselines in robust accuracy. The robust accuracy results obtained by TRADES are always higher than those obtained by PGD-AT through the training process. This agrees with the conclusions for deep networks \cite{zhang2019theoretically,croce2021robustbench}. It is also observed that the PGD-AT framework leads to faster robustness overfitting, while this effect is not obvious in the TRADES experiments. Finally, it is noted that DEQ models are by default trained with Adam optimizer \cite{adam-optimizer}, zero weight decay, and learning rate cosine decay \cite{lr-cosine-decay}. This differs from the common practice in training robust deep networks, where the optimizer is usually SGD, with weight decay and early stopping after the first time the learning rate is decayed \cite{zhang2019theoretically,rice2020overfitting}. While these tricks have proven to be crucial in adversarial training \cite{pang2020bag}, we failed to implement similar techniques in DEQ training: For example, our initial experiments show that when setting weight decay to be 5e-4, the loss becomes NaN after about 20 training epochs with PGD-AT. We leave more effective and efficient adversarial training for DEQ models as future work. 

\begin{table}[t]
    \centering
    \caption{Statistics of the $P$ values calculated with all of the $144$ types of intermediate state attacks according to Eq. (\ref{eq-analysis-P}).}
    \label{tab:analysis_p}
    \resizebox{0.99\linewidth}{!}{\begin{tabular}{llccc}
        \toprule
        \textsc{AT} & \textsc{Method} & \textsc{Avg} & \textsc{Min} & \textsc{Max} \\
        \midrule
        \multirow{2}{*}{\textsc{PGD-AT}} & \textsc{\citet{yang2022acloser}} & 78.26\% & 72.98\% & 85.53\% \\
        ~ & ~\textsc{+ Sec. ~\,\ref{sec3:3-training-random}} & 78.39\% & 68.79\% & 83.77\% \\
        \midrule
        \multirow{2}{*}{\textsc{TRADES}} & \textsc{\citet{yang2022acloser}} & 81.46\% & 75.88\% & 87.99\% \\
        ~ & ~\textsc{+ Sec. ~\,\ref{sec3:3-training-random}} & 82.34\% & 76.63\% & 88.58\% \\
        \bottomrule
    \end{tabular}}
\vspace{-10pt}
\end{table}

\begin{table}[t]
    \centering
    \caption{Statistics of the $\Delta H$ values calculated with all of the $144$ types of intermediate state attacks according to Eq. (\ref{eq-analysis-dH}).}
    \label{tab:analysis_dH}
    \resizebox{0.99\linewidth}{!}{\begin{tabular}{llccc}
        \toprule
        \textsc{AT} & \textsc{Method} & \textsc{Avg} & \textsc{Min} & \textsc{Max} \\
        \midrule
        \multirow{2}{*}{\textsc{PGD-AT}} & \textsc{\citet{yang2022acloser}} & -0.1948 & -0.2133 & -0.1156 \\
        ~ & ~\textsc{+ Sec. ~\,\ref{sec3:3-training-random}} & -0.1027 & -0.1337 & -0.0612 \\
        \midrule
        \multirow{2}{*}{\textsc{TRADES}} & \textsc{\citet{yang2022acloser}} & -0.1174 & -0.1368 & -0.0769 \\
        ~ & ~\textsc{+ Sec. ~\,\ref{sec3:3-training-random}} & -0.1179 & -0.1445 & -0.0836 \\
        \bottomrule
    \end{tabular}}
\end{table}

\subsection{Quantitative Analysis for Prediction Entropy}

\looseness=-1 As demonstrated in Sec. \ref{sec3:2-testing-entropy-reduction}, progressively reducing the predicted entropy of the input is beneficial to the regulation of its corresponding neural dynamics. While the improved performances (shown in different tables) have proved the effectiveness of the method, in this section, we conduct additional analysis to quantitatively compare the prediction entropies between a clean input and its perturbed counterpart.

We conduct the comparison using different types of attacks and adversarial training configurations. In our analysis, the adversarial inputs are generated by PGD-10 with the $144$ different intermediate attacks described in the setup of Sec. \ref{sec:experiments}. Under a certain attack, for each clean input $\bx_j$ in the validation set $\mathcal{V}$ and its perturbed counterpart $\tilde{\bx}_j$, their corresponding equilibrium states are denoted as $\bz_j^{\star}$ and $\tilde{\bz}_j^{\star}$.

We propose two metrics for the quantitative comparison of prediction entropy. For a certain attack, we calculate $P$, the percentage of clean inputs with prediction entropy lower than their perturbed counterparts. Formally,
\begin{equation}
    P = \frac{1}{|\mathcal{V}|} \sum_{j} {\LARGE \mathds{1}}(H(\bz_j) < H(\tilde{\bz}_j^{\star})) \times 100 \%,
\label{eq-analysis-P}
\end{equation}
where $H(\cdot)$ is the prediction entropy defined in Eq. (\ref{eq:prediction-entropy}). 

We also calculate $\Delta H$, the difference of the prediction entropy averaged in the validation set between the clean inputs and their perturbed counterparts:
\begin{equation}
    \Delta H = \frac{1}{|\mathcal{V}|} \sum_{j}(H(\bz_j) - H(\tilde{\bz}_j^{\star})).
\label{eq-analysis-dH}    
\end{equation}

We calculate a pair of $P$ and $\Delta H$ for each attack. For the $144$ $P$s and the $144$ $\Delta H$s, we list their average, their minimum, and their maximum value in Table \ref{tab:analysis_p} and Table \ref{tab:analysis_dH}. According to the $P$ statistics in Table \ref{tab:analysis_p}, an average of over three-quarters of clean inputs have lower prediction entropy than their perturbed counterparts under all types of attacks. Furthermore, the $\Delta H$ statistics in Table \ref{tab:analysis_dH} show that the averaged prediction entropy of clean inputs in the validation set is always less than that of the perturbed inputs. These two quantitative findings again verify the viability of our entropy reduction method in Sec. \ref{sec3:2-testing-entropy-reduction}.

\section{Related Work}

\subsection{Training-Time Adversarial Defense}
\looseness=-1 Of all the training-time adversarial defense approaches, adversarial training has proven to be the most practical and effective technique for improving adversarial robustness \cite{athalye2018obfuscated}. However, AT only regulates the input-output behavior of neural models, leaving the internal neural dynamics under-supervised. The most related effort of explicit regulations along the entire neural dynamics is interval bound propagation (IBP) \cite{ibp,ibp-nlp,crown-ibp}. IBP is a technique from the certificated robustness field that envelopes the neural dynamics of deep networks with layer-wise linear functions for robustness guarantees. However, the complicated training procedure and the lack of scalability hinder its practical use. In our work, we exploit the structural uniquenesses of DEQs to impose explicit regulations along their neural dynamics.

\subsection{Test-Time Adversarial Defense}
Recently, several test-time defense techniques have been proposed to exploit additional computes during inference time for robustness improvement \cite{test-time-defense-1,test-time-defense-2,test-time-defense-3,test-time-defense4,test-time-defense5}. \citet{Croce2022testtime} conduct a thorough evaluation for test-time defenses. Our method is different from the previous works, as they focus on traditional deep networks, and usually calibrate only the output behavior of the model. In contrast, we progressively update the input along the forward pass to mount the neural dynamics of DEQ models to correct ``orbits'' in Sec. \ref{sec3:2-testing-entropy-reduction}. The fundamental difference between our work and prior arts originates from the special design of DEQ models, as they directly cast the forward process as solving the fixed-point equation iteratively. 

\subsection{Dynamical System Perspective for Neural Models}
\citet{eproposal} first proposes to interpret deep networks from a dynamical system perspective, which draws the connection between residual networks and the solution of an ODE with the forward Euler method. Since then, multiple types of novel neural models have been proposed, which directly model a dynamical system in their forward pass. Among them, neural ODEs \cite{neuralode} are integrated with continuous ODE, while DEQ models \cite{deq,mdeq} are instantiated by discrete fixed-point iteration systems.

Several efforts have been made in designing robust neural ODEs by drawing inspiration from control theory  \cite{lu-adv,yang2020interpolation,kang2021robustode}. By comparison, for DEQ models, Jacobian regularization is proposed in \cite{jacobian-deq} by regulating only the equilibrium state to improve training stability instead of robustness. In our work, we explicitly regulate the entire neural dynamics of DEQ models to improve the adversarial robustness.

\section{Conclusion}
In this work, we propose to reduce the prediction entropy of intermediate states along the DEQ neural dynamics with progressive input updates. We also randomly select intermediate states to compute the loss function during adversarial training of DEQ models. Our work significantly outperforms previous works on improving DEQ robustness and even surpasses strong deep network baselines. Our work sheds light on explicitly regulating DEQ and other neural models from the perspective of neural dynamics.

In the future, we will continue to exploit the special properties (single layer, fixed-point structure, neural dynamics, etc.) of DEQ models to design tailored adversarial defense strategies. We will also investigate the relationship between our methods and the inexact/approximated gradient proposed for implicit models \cite{fung2022jfb,deq-phantom-grad}. We also leave the validation of our methods on larger benchmarks as future work. 

\section*{Acknowledgment}

We thank all of the anonymous reviewers for their constructive suggestions. This work was supported by the National Key R\&D Program of China (2022ZD0160502) and the National Natural Science Foundation of China (No. 61925601, 62276152, 62236011).

\bibliography{icml23_yzh}
\bibliographystyle{icml2023}

\newpage
\appendix
\onecolumn
\section{Experiment Details}\label{app:exp-details}

We follow to adopt the DEQ-Large architecture used in \citet{yang2022acloser}, which has similar parameter counts as ResNet-18. More specifically, the DEQ cell we use is the multiscale DEQ-Large with $4$ scales. The numbers of head channels for each scale are $14$, $28$, $56$, and $112$. The numbers of channels for each scale are $32$, $64$, $128$, and $256$. The channel size of the final layer is $1680$. We use Broyden's method as the black-box solver, with $N=8$ forward solver iterations. 

Following \citet{mdeq,jacobian-deq}, the model is trained with Adam optimizer: the initial learning rate is $0.001$ with cosine decay; the Nesterov momentum is $0.98$; the weight decay is $0$. We follow \citet{mdeq,jacobian-deq} and \citet{yang2022acloser} to set batch size as $96$. We pretrain the DEQ models with the truncated deep networks (with standard training) for a good initialization of $\theta$ for $16000$ steps, and then conduct adversarial training, until the total of $210$ epochs training finishes. Following \citet{yang2022acloser}, we use the unrolling-based phantom gradient to train the DEQ models with $5$ unrolling steps \cite{deq-phantom-grad}. For PGD-AT training, the perturbation range is $\epsilon=8/255$; the step size is $\alpha=2/255$; the number of PGD steps during training is 10. Additionally, for TRADES training, the $1/\lambda=6$. 

We select the best model weight with the top robust accuracy on the CIFAR-10 validation set under ready-made PGD-10 (at the final state after unrolling). After obtaining the weight, we follow \citet{yang2022acloser} to leverage early-state defenses. In our experiments, we always use the last but one intermediate state as the early state for robustness evaluation. We leverage the unrolled intermediates method for intermediate attacks, as they form consistently stronger attacks than the simultaneous adjoint method. For the Sec. \ref{sec3:2-testing-entropy-reduction} method, we set $\beta=2/255$, $T_f=2$, and $R=10$ in our main experiments. All experiments are conducted on a single NVIDIA 3090 GPU. Our code is available at \url{https://github.com/minicheshire/DEQ-Regulating-Neural-Dynamics}.

\section{The Entropy in Dynamical Systems}\label{app:entropy-dynamical-systems}

In this section, we provide a brief introduction to the entropy in dynamical systems. This introduction heavily relies on \citet{young2003entropy}. We draw connections between the concepts of the entropy in dynamical systems with DEQ models.

\begin{definition}
Let $(\mathcal{X}, d)$ be a metric space, and $\mu$ be a probability measure on it. Let $\alpha=\{X_1, \cdots, X_C\}$ be a finite partition of $\mathcal{X}$. Then the entropy of the partition is defined as
\begin{equation}\label{df-1}
    H(\alpha) = H(\{X_1, \cdots, X_C\}) = -\sum_{i=1}^{C} \mu(X_i)\log\mu(X_i).
\end{equation}
\end{definition}
\textbf{Implications:} Interpreting $\mathcal{X}$ as the space of $\bz^{[t]}$ in DEQ models, the partition of $\mathcal{X}$ is naturally constructed by the $C$ classes in a classification task. $\mu$ can therefore be instantiated by the classification head in DEQ models.

\begin{definition}
Let $(f, \mu)$ be an ergodic discrete dynamical system, with $f:\mathcal{X} \to \mathcal{X}$ is the mapping from $\mathcal{X}$ to $\mathcal{X}$. We define the $\alpha$-address of the $n$-orbit starting at $x \in \mathcal{X}$ as
\begin{equation}
    \bigvee_{i=0}^{n-1}f^{-i}\alpha = \{x: x \in X_{i_0}, fx \in X_{i_1}, \cdots, f^{n-1}x \in X_{i_{n-1}} \}
\end{equation}
for some $(i_0,i_1, \cdots, i_{n-1})$. The metric entropy of $f$ with partition $\alpha$ is then defined as 
\begin{equation}
    h_\mu(f, \alpha) = \lim_{n \to \infty}\frac{1}{n} H(\bigvee_{i=0}^{n-1}f^{-i}\alpha).    
\end{equation}
\end{definition}
\textbf{Implications:} The $n$-orbit of $\{x, fx, \cdots, f^{n-1}x \}$ coincides with the $n$-length neural dynamics of DEQ models we defined in Eq. \ref{eq-4-neural-dynamics}. The $\alpha$-address $\bigvee_{i=0}^{n-1}f^{-i}\alpha$ corresponds to all $x$s that induces an $n$-length orbit traversing different subspaces of the partitions with the mapping $f$. Of all $x$s that comprise the $\alpha$-address of the $n$-orbit, $H(\bigvee_{i=0}^{n-1}f^{-i}\alpha)$ characterizes the entropy of the address and roughly reflects how ``chaotic'' the $f$ is under the partition $\alpha$ in the space $\mathcal{X}$. While $h_\mu(f, \alpha)$ is defined by taking the limit of $n \to \infty$, $\lim_{n\to \infty} f^{n-1}x$ is exactly the equilibrium point of $f$ starting with $x$. This suggests that $h_\mu(f, \alpha)$ is defined over the entire orbit and concerned about the equilibrium state behavior. 

\begin{theorem}
For $x \in \mathcal{X}$, $n \in \mathbb{N}^{+}$, $\epsilon>0$, define
\begin{equation}
    B(x,n,\epsilon) := \{ y \in \mathcal{X}: d(f^{i}x, f^{i}y) < \epsilon, 0\le i < n \}.
\end{equation}
With $\mu$ defined over partition $\alpha$, we have
\begin{equation}
    h_\mu (f, \alpha) = \limsup_{n\to\infty} -\frac{1}{n} \log \mu B(x,n,\epsilon).
\label{eqeqeqeq}
\end{equation}
Specifically, if $\mathcal{X}$ is a compact Riemann manifold with $\mu$ equivalent to the Riemannian measure, we have 
\begin{equation}
    h_\mu (f, \alpha) = \sum_i \max(\lambda_i, 0),
\label{eqeqeqeq-lyapunov}    
\end{equation}
where $\lambda_i$s are the Lyapunov exponents of $(f,\mu)$.
\end{theorem}
\begin{proof}
The proof is given by \citet{Pesin_1977} and \citet{brin1983local}.
\end{proof}
\textbf{Implications}: $B(x,n,\epsilon)$ indicates the ``amount'' of the neighbors of $x$ that induce the trajectories close to that of $x$ with $f$. This definition shares a similar idea with our derivation in Sec. \ref{sec3:1-overview}, as we consider the deviation of the neural dynamics under clean or perturbed inputs. We want all the perturbed inputs to lie within $B(x,n,\epsilon)$ in Eq. \ref{eqeqeqeq}, in this way all the orbits are regulated. This corresponds with an increase in $\mu B(x,n,\epsilon)$, and from Eq. \ref{eqeqeqeq} we realize that reducing the entropy $h_\mu (f, \alpha)$ achieves this. From Eq. (\ref{eqeqeqeq-lyapunov}) we know that reducing $h_\mu (f, \alpha)$ is also equivalent to reducing the sum of the Lyapunov exponents of the system. Smaller Lyapunov exponents imply more stable dynamical systems \cite{lyapunov-1992}, which is equivalent to more robust neural models as demonstrated by \citet{yang2020interpolation,kang2021robustode}. Such correlations motivate us to observe and reduce prediction entropy along the neural dynamics in DEQ models.

\section{Additional Ablation Studies}\label{app:additional-ablation}

\subsection{Effect of $\lambda$ in Eq. (\ref{eq:unrolled-intermediates}) of Intermediate-State Attacks}\label{app:additional-ablation-lambda}

In this section, we compare the effect of setting $\lambda$ to be $1.0$ or $0.5$ in Eq. (\ref{eq:unrolled-intermediates}) as the intermediate-state attacks. We use the ``TRADES+Secs.\ref{sec3:2-testing-entropy-reduction} \& \ref{sec3:3-training-random}'' defense method for this study.
Similar to Fig. \ref{fig:attack-along-neural-dynamics}-(a), we plot the lowest accuracy results under the attacks at each intermediate states along the neural dynamics in Fig. \ref{fig:FMunroll-fig}. It is obvious that $\lambda=0.5$ builds consistently stronger attacks than $\lambda=1.0$.
\begin{figure}[h]
    \centering
    \includegraphics[width=0.5\textwidth]{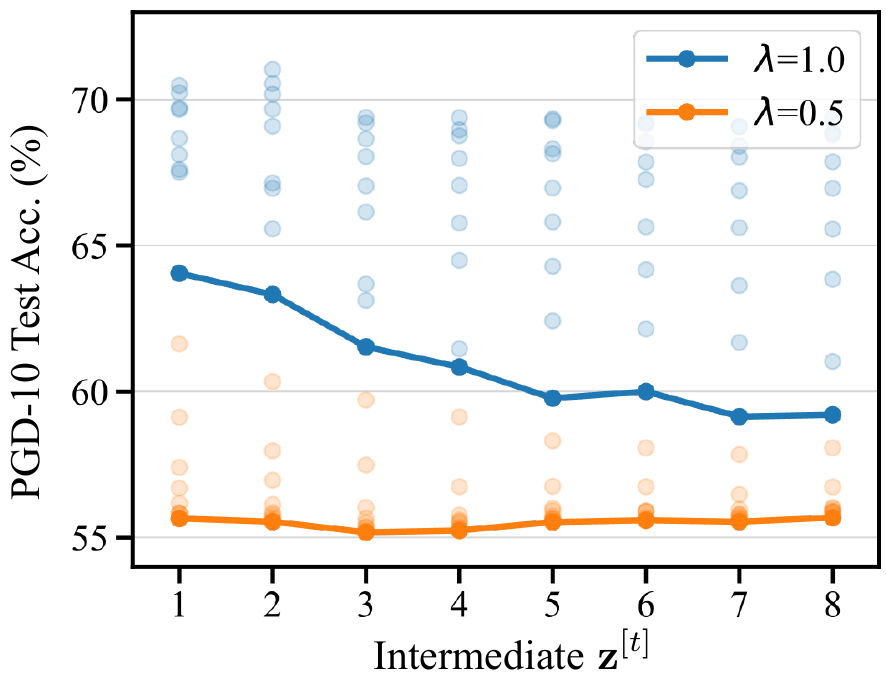}
    \caption{Performance (\%) of the ``TRADES+Secs.\ref{sec3:2-testing-entropy-reduction} \& \ref{sec3:3-training-random}'' defense under all intermediate-state attacks along the neural dynamics, with $\lambda=1.0$ and $\lambda=0.5$. $\lambda=0.5$ forms the intermediate-state attacks consistently stronger than those with $\lambda=1.0$.}
    \label{fig:FMunroll-fig}
\end{figure}

\begin{table}[h]
    \centering
    \caption{The strongest attacks in Eq. (\ref{eq:unrolled-intermediates}) at each intermediate state along the neural dynamics.}
    \label{tab:FMunroll-tab}    
    \begin{tabular}{cccccccccc}
    \toprule
        \multirow{2}{*}{$\lambda=1.0$} & $(i,K_a)$ & $(1,1)$ & $(2,1)$ & $(3,1)$ & $(4,1)$ & $(5,1)$ & $(6,1)$ & $(7,1)$ & $(8,1)$ \\
    \cmidrule{2-10}
        ~ & \textsc{PGD} & 64.05 & 63.32 & 61.53 & 60.84 & 59.77 & 60.00 & 59.14 & 59.20  \\
    \midrule    
        \multirow{2}{*}{$\lambda=0.5$} & $(i,K_a)$ & $(1,8)$ & $(2,7)$ & $(3,5)$ & $(4,5)$ & $(5,5)$ & $(6,5)$ & $(7,4)$ & $(8,5)$ \\
    \cmidrule{2-10}        
        ~ & \textsc{PGD} & 55.65 & 55.54 & 55.18 & 55.24 & 55.53 & 55.59 & 55.54 & 55.68 \\
    \bottomrule
    \end{tabular}
\end{table}

We further list the strongest attacks at each intermediate state along the neural dynamics in Table \ref{tab:FMunroll-tab}. For $\lambda=1.0$, the best setting of $K_a$ for each state $\bz^{[t]}$ is $1$. The overall strongest attack is formed by $(7,1)$. It is noted that $\bz^{[7]}$ is one iteration away from $\bz^{\star} = \bz^{[8]]}$. On the contrary, for $\lambda=0.5$, the strongest attack is in the middle of the neural dynamics, with $(i,K_a)=(3,5)$, and $K_a$ is larger than $1$ under $\lambda=0.5$. It is inferred that $\lambda=0.5$ constructs intermediate attacks with more accurate gradient estimation: the estimated gradients become more accurate as the unrolling step $K_a$ becomes larger, so that the corresponding attacks are stronger.

\subsection{Effect of the Jacobian Regularization Factor}\label{app:additional-ablation-jacreg}

\begin{figure}[h]
    \centering
    \includegraphics[width=0.8\textwidth]{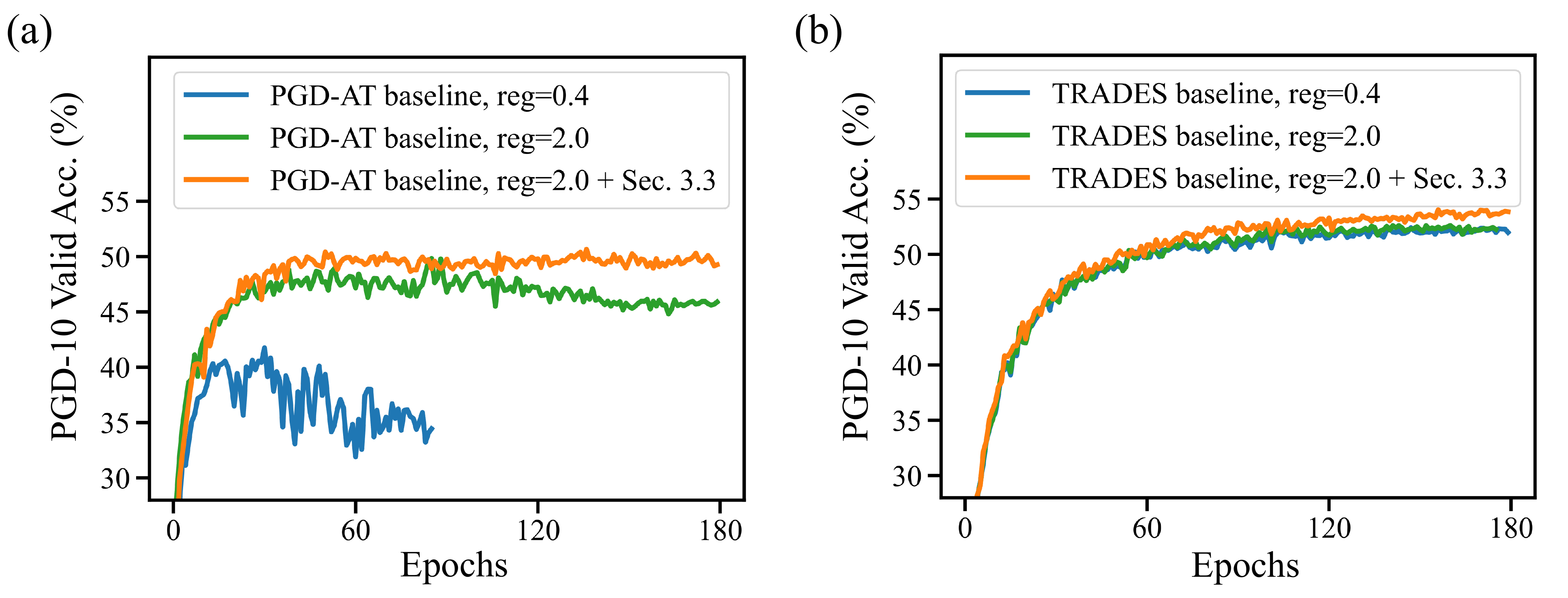}
    \caption{Comparisons between different Jacobian regularization factors during the training phase of baseline methods.}
    \label{fig:jac-reg-comparisons}
\end{figure}

We also tried to train DEQ models by setting the Jacobian regularization factor proposed by \citet{jacobian-deq} as $0.4$, following \citet{yang2022acloser}. However, shown in Fig. \ref{fig:jac-reg-comparisons}, we found that under PGD-AT, the training becomes unstable with suboptimal robustness performance. After increasing the factor to $2.0$, we found the performance during training is improved. We suspect that a larger Jacobian regularization factor is required to stabilize the training PGD-AT. In contrast, when trained with TRADES, we found that the empirical difference between setting the factor to be $0.4$ or $2.0$ is marginal. It is therefore inferred that TRADES might implicitly impose more regularization during the adversarial training process, which might also account for the superiority of TRADES experiments.

\clearpage
\section{Running Time Analysis}\label{app:time-cost}

\looseness=-1 In this section, we compare the time cost of our methods with the vanilla adversarial training baseline for DEQ models. We first compare the training-time methods: the vanilla baseline and the Section \ref{sec3:3-training-random} method. The comparison is shown in Table \ref{tab:training-time-comparison}.

\begin{table}[h]
    \centering
    \caption{Training speed comparison (Samples/s) between the vanilla adversarial training baseline \cite{yang2022acloser} and the Sec. \ref{sec3:3-training-random} method.}
    \label{tab:training-time-comparison}    
    \begin{tabular}{llc}
    \toprule
        \textsc{AT Framework} & \textsc{Method} & \textsc{Training Speed (Samples/s)} \\
    \midrule
        \multirow{2}{*}{\textsc{PGD-AT}} & \citet{yang2022acloser} & 30.2 \\
        ~ & + \textsc{Ours (Sec. \ref{sec3:3-training-random})} & \textbf{39.9} \\
    \midrule
        \multirow{2}{*}{\textsc{TRADES}} & \citet{yang2022acloser} & 24.0 \\
        ~ & + \textsc{Ours (Sec. \ref{sec3:3-training-random})} & \textbf{30.6} \\
        \bottomrule
    \end{tabular}
\end{table}
According to Table \ref{tab:training-time-comparison}, our Sec.\ref{sec3:3-training-random} method is faster than the vanilla adversarial training. This is because we use random intermediate states for loss computation in Sec. \ref{sec3:3-training-random}. In this way, the fixed-point solver in the forward process runs fewer than $N$ iterations. In contrast, the solver always runs for $N$ iterations in the baseline.

Next, we compare the running speed among different settings in the Sec. \ref{sec3:2-testing-entropy-reduction} method during inference.   
\begin{table}[h]
    \centering
    \caption{Inference speed comparison (Samples/s) with different settings in the Sec. \ref{sec3:2-testing-entropy-reduction} method during inference. $R=0$ indicates the baseline method without Sec. \ref{sec3:2-testing-entropy-reduction}. }   
    \begin{tabular}{l|c|rrrr}
        \toprule
           ~    & $R=0$ & $R=1$ & $R=5$ & $R=10$ & $R=20$  \\
        \midrule   
        $T_f=1$ & \multirow{4}{*}{94.9} & 32.1 & 11.8 & 8.6 & 6.5 \\
        $T_f=2$ & ~ & 42.6 & 20.5 & 11.7 & 7.5 \\
        $T_f=4$ & ~ & 48.6 & 29.9 & 18.6 & 12.8 \\
        $T_f=8$ & ~ & 49.6 & 33.1 & 26.6 & 17.6 \\
        \bottomrule
    \end{tabular}
    \label{tab:sec32-running-time}
\end{table}

Reported in Table \ref{tab:sec32-running-time}, the time consumption of Sec. \ref{sec3:2-testing-entropy-reduction} method is roughly proportional to $R$ and inversely proportional to $T_f$. This coincides with our Algorithm \ref{alg:entropy_reduction}. Comparing Table \ref{tab:sec32-running-time} with Table \ref{tab:ablation-effect-sec32}, we set $R=10$ and $T_f=2$ to achieve the trade-off between inference speed and robustness performance.


\end{document}